\documentclass[11pt]{article}

\usepackage[margin=1in]{geometry}
\usepackage{amsmath,amssymb,amsthm}
\usepackage{booktabs}
\usepackage[T1]{fontenc}
\usepackage[hidelinks]{hyperref}
\usepackage{tikz}
\usetikzlibrary{arrows.meta}

\newtheorem{theorem}{Theorem}
\newtheorem{lemma}[theorem]{Lemma}
\newtheorem{proposition}[theorem]{Proposition}

\newcommand{\R}{\mathbb R}
\newcommand{\SO}{\mathrm{SO}}
\newcommand{\tr}{\operatorname{tr}}

\title{Sharp Reconstruction Bounds for Autoencoders Using the Same Forward Map}

\author{
Patricia Medina\thanks{
New York City College of Technology, CUNY, Brooklyn, NY, USA;
The Graduate Center, CUNY, New York, NY, USA.
}
\and
Hy P. G. Lam\thanks{
Department of Mathematical Sciences,
Worcester Polytechnic Institute, Worcester, MA, USA.
}
}

\date{}
\begin{document}
\maketitle

\begin{abstract}

We study reconstruction in autoencoders that apply the same forward map before and after setting the observed coordinates to zero. For equal odd input and hidden dimensions $d\geq3$, among orientation-preserving diffeomorphisms whose Jacobian singular values lie in $[m,M]$, we show that the least uniform reconstruction-derivative error is
$\max\{1-M(M-m)/2,0\}$, with affine maps attaining this sharp bound at every prescribed depth. A translated radial rotation can nevertheless reconstruct any prescribed ball exactly with singular values arbitrarily close to one, motivating additional conditions for a finite-data bound.
We test this prediction on a 798,452-point terrestrial LiDAR forest scan. At input scale $0.05$, the mean theoretical bound is $0.155$, about $84\%$ of the mean normalized training error $0.185$ across four spatial regions, two depths, and three seeds. At this scale, adding one hidden coordinate reduces the mean reconstruction error below $6\times10^{-6}$.
\end{abstract}


\section{Introduction}

Autoencoders learn representations from which an input can be reconstructed,
typically through distinct encoder and decoder maps. Dynamical system
autoencoders (DSAEs) \cite{he2024} impose a different structure: the same
learned forward map is iterated during both encoding and decoding. Starting
from an observed input with zero hidden coordinates, the map evolves the full
state, the observed coordinates are then reset to zero, and the same map is
applied again to reconstruct the input. Successful reconstruction therefore
requires information to be transferred from the observed coordinates into the
hidden coordinates before the reset and then from the hidden coordinates back
into the observed coordinates afterward.

This raises a basic geometric question: does preservation of full-state
distances protect the information needed for reconstruction after the observed
coordinates are removed? In general, it does not. For example, leaving the
full state unchanged preserves all distances, but starting from $(x,0)$ also
leaves every hidden coordinate equal to zero. After the observed coordinates
are reset, no information about $x$ remains in the hidden state. Thus the
relevant issue is not only whether the full-state map preserves information,
but whether it transfers that information between the observed and hidden
coordinates in the directions required by the architecture.

Related architectures control the geometry of learned transformations in
different ways. Orthogonal recurrent networks and invertible residual networks
bound or preserve properties of the full map
\cite{arjovsky2016,lezcano2019,behrmann2019,haber2018}. Geometric
autoencoders and decoder constraints concern feature distances
\cite{nazari2023,zhan2026}. Changes of hidden coordinates leave linear
reconstruction unchanged \cite{kunin2019}, while additional assumptions can
restrict this ambiguity \cite{nelson2026}. The setting considered here is
different: encoding and decoding apply the same map in the same direction,
rather than a map and its inverse or transpose, with a reset separating the
two stages.

The key geometric mechanism is that effective transfer in both directions
between observed and hidden coordinates can require unequal stretching of the
full state. Lemma~\ref{spread} makes this precise by relating the
off-diagonal transfer blocks of the Jacobian to the gap between its largest
and smallest singular values. Building on this restriction, we obtain a sharp
reconstruction bound for equal odd observed and hidden dimensions $d\geq 3$.
Among orientation-preserving diffeomorphisms whose Jacobian singular values
lie in $[m,M]$, the smallest possible uniform reconstruction-derivative error
is
\[
\left[1-\frac{M(M-m)}{2}\right]_+.
\]
Explicit affine maps attain this bound at every prescribed iteration depth.
Consequently, exact reconstruction of every input under these assumptions is
possible precisely when
\[
M(M-m)\geq 2.
\]
In particular, a rotation has $m=M=1$ and cannot provide exact reconstruction
when the observed and hidden dimensions are equal and odd, whereas one
additional hidden coordinate removes this specific parity obstruction.

The global result alone does not imply a positive reconstruction error on an
arbitrary bounded dataset. Proposition~\ref{local} gives additional
conditions under which the derivative restriction yields a lower bound on
average data reconstruction error, controlling the hidden value at the origin
and the variation of the reconstruction derivative.
Theorem~\ref{ball} shows why such conditions matter: a translated radial
rotation can have Jacobian singular values arbitrarily close to one while
reconstructing every input in a prescribed ball exactly. The translation
separates the states used for encoding from those used for decoding, allowing
the same nonlinear map to behave differently in the two regions. We also
characterize exact reconstruction when the full map is globally nonexpansive;
in this case, the hidden representation is forced to be an affine isometry.

Finally, we test the finite-data bound on a 798,452-point terrestrial LiDAR
forest scan using only its three spatial coordinates. This choice is
deliberate: with input dimension $d=3$ and equal hidden width $r=3$, the
experiment lies directly in the odd-dimensional setting of the reconstruction
theorem. Adding one hidden coordinate gives $r=4$ and tests the predicted
removal of this specific obstruction. Across four spatially held-out regions,
two iteration depths, and three seeds, the mean lower bound at input scale
$0.05$ is $0.155$, approximately $84\%$ of the mean normalized training
reconstruction error $0.185$. At the same scale, models with one additional
hidden coordinate have mean reconstruction error below $6\times10^{-6}$.

\section{The optimal global error}
Let $x\in\R^d$ and $y\in\R^r$ denote the observed and hidden coordinates. Define $Ex=(x,0)$, $P(x,y)=(0,y)$, $Q(x,y)=x$, and $H(x,y)=y$. For a map $F$ from $\R^{d+r}$ to itself and an integer $K\ge1$, let $G=F^K$ denote its $K$-fold composition and put
\begin{equation}\label{model}
 h(x)=HGEx,\qquad g(y)=QG(0,y),\qquad f=g\circ h.
\end{equation}
Then $f=QGPGE$ is the reconstruction. Inner products and vector norms are Euclidean. For square real $A$, the singular values $\sigma_j(A)$ are the square roots of the eigenvalues of $A^\top A$. The least and greatest satisfy
\[
 \sigma_{\min}(A)=\min_{\|v\|=1}\|Av\|,\qquad
 \sigma_{\max}(A)=\max_{\|v\|=1}\|Av\|.
\]
Here $A^\top$ denotes transpose, $I$ is the identity, and $\tr$ sums diagonal entries. We use $\|A\|=\sigma_{\max}(A)$ and $\|A\|_F^2=\tr(A^\top A)$. Write $DT$ for the Jacobian and $\operatorname{Lip}(T)$ for the least $L$ such that $\|T(u)-T(v)\|\le L\|u-v\|$. Nonexpansive means $1$-Lipschitz.

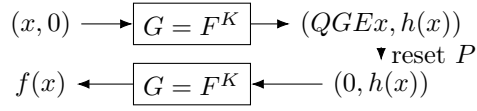
\begin{figure}[t]
\centering
\begin{tikzpicture}[>=Latex,font=\small]
\node (a) at (0,0) {$(x,0)$};
\node[draw,minimum width=1cm,minimum height=.5cm] (b) at (2,0) {$G=F^K$};
\node (c) at (4.5,0) {$(QGEx,h(x))$};
\node (d) at (4.5,-.8) {$(0,h(x))$};
\node[draw,minimum width=1cm,minimum height=.5cm] (e) at (2,-.8) {$G=F^K$};
\node (f) at (0,-.8) {$f(x)$};
\path[->] (a) edge (b) (b) edge (c)
 (c) edge node[right]{reset $P$} (d)
 (d) edge (e) (e) edge (f);
\end{tikzpicture}
\caption{At $x=0$, the two inputs to $G$ coincide exactly when $h(0)=0$.}
\end{figure}

For this section, $r=d$, $d\ge3$ is odd, and $G$ is a continuously differentiable ($C^1$) diffeomorphism with $\det DG(u)>0$ satisfying
\begin{equation}\label{bounds}
 0<m\le\sigma_j(DG(u))\le M\qquad\text{for all }u,j.
\end{equation}
Set
\begin{equation}\label{floor}
 \delta_0=\left[1-\frac{M(M-m)}2\right]_+,
 \qquad [t]_+=\max\{t,0\}.
\end{equation}
\begin{lemma}\label{spread}
Let $R=\left(\begin{smallmatrix}D&B\\C&T\end{smallmatrix}\right)\in\R^{2d\times2d}$ be invertible, with invertible blocks $B,C\in\R^{d\times d}$. If $\det R$ and $(-1)^d\det(BC)$ have opposite signs, then
\[
 \sigma_{\max}(R)-\sigma_{\min}(R)
 \ge2\min\{\sigma_{\min}(B),\sigma_{\min}(C)\}.
\]
\end{lemma}
\begin{proof}
Put $X=\left(\begin{smallmatrix}0&B\\C&0\end{smallmatrix}\right)$, $V=\operatorname{diag}(D,T)$, $J=\operatorname{diag}(I,-I)$, and $Y=X^{-1}V$, where $\operatorname{diag}$ lists diagonal blocks. Since $\det X=(-1)^d\det(BC)$, we have $\det(I+Y)<0$. Nonreal conjugate eigenvalues contribute $|1+\lambda|^2>0$ to this determinant. Since $YJ=-JY$, real eigenvalues occur in pairs $\lambda,-\lambda$, contributing $1-\lambda^2$. Negativity therefore requires a real eigenvalue $\mu>1$. Choose $w\ne0$ with $Vw=\mu Xw$. Then
\[
 Rw=(\mu+1)Xw,\qquad RJw=(\mu-1)JXw.
\]
Consequently,
\[
\begin{aligned}
\sigma_{\max}(R)-\sigma_{\min}(R)
&\ge\frac{\|Rw\|-\|RJw\|}{\|w\|}\\
&=2\frac{\|Xw\|}{\|w\|}\ge2\sigma_{\min}(X).
\end{aligned}
\]
Finally, $X^\top X=\operatorname{diag}(C^\top C,B^\top B)$.
\end{proof}

For any $R$ satisfying \eqref{bounds} with positive determinant,
\begin{equation}\label{matrixfloor}
 \|BC-I\|\ge\delta_0.
\end{equation}
If $q=\|BC-I\|<1$, the path $I+t(BC-I)$ is invertible for $0\le t\le1$, so $\det(BC)>0$. Moreover,
\[
\|C^{-1}\|=\|(BC)^{-1}B\|\le\frac M{1-q},\quad
\|B^{-1}\|\le\frac M{1-q}.
\]
Lemma~\ref{spread} gives $M-m\ge2(1-q)/M$. If $q\ge1$, \eqref{matrixfloor} follows from $\delta_0\le1$.

\begin{theorem}\label{global}
For the nonlinear class in \eqref{bounds},
\begin{equation}\label{sharp}
 \inf_G\ \sup_{x\in\R^d}\|Df(x)-I\|=\delta_0.
\end{equation}
The infimum is attained by an affine map. The same answer holds when $G$ is required to be $F^K$ for any prescribed $K\ge1$.
\end{theorem}
\begin{proof}
Let $\delta=\sup_x\|Df(x)-I\|$. It suffices to treat $\delta<1$. Integration of $Df-I$ along a segment gives
\[
 (1-\delta)\|x-y\|\le\|f(x)-f(y)\|.
\]
The path $I+t(Df(x)-I)$ gives $\det Df(x)>0$. The lower distance bound gives injectivity and $\|f(x)\|\ge(1-\delta)\|x\|-\|f(0)\|$, hence properness. Thus the image is closed; the inverse function theorem makes it open. Since $\R^d$ is connected, $f(\R^d)=\R^d$.

Since $h$ and $g$ are $M$-Lipschitz, $h$ has lower distance bound $(1-\delta)/M$. Thus $Dh$ is invertible, and the same open-and-closed argument makes $h$ a diffeomorphism. Now $g=f\circ h^{-1}$ is a diffeomorphism. The chain rule gives
\[
 \sigma_{\min}(Dh),\ \sigma_{\min}(Dg)\ge(1-\delta)/M
\]
at every argument, using surjectivity of $h$ for the bound on $Dg$. The determinant signs of $Dh$ and $Dg$ are constant and equal, since $\det Df>0$. The off-diagonal blocks of $DG(0)$ are $Dg(0)$ and $Dh(0)$. Lemma~\ref{spread} therefore yields $M-m\ge2(1-\delta)/M$. This does not require $h(0)=0$.

For attainment, put $\beta=\min\{1,M(M-m)/2\}$, $c=\beta/M$, and $a=M-c$. On $(x_1,x_2,y_1,y_2)$ use
\begin{equation}\label{fourblock}
 R_4=\begin{pmatrix}0&0&M&0\\a&0&0&c\\c&0&0&a\\0&M&0&0\end{pmatrix}.
\end{equation}
On each remaining pair $(x_j,y_j)$ use
\[
 S=\begin{pmatrix}\sqrt{M^2-\beta}&\sqrt\beta\\
 \sqrt\beta&-\sqrt{M^2-\beta}\end{pmatrix}.
\]
In $R_4$, these blocks are $\operatorname{diag}(M,c)$ and $\operatorname{diag}(c,M)$. Their product is $\beta I_2$, giving $BC=\beta I$ for $R$. Since $a+c=M$, for $v\in\R^4$,
\[
\begin{aligned}
 \|R_4v\|^2={}&M^2(v_2^2+v_3^2)+\tfrac12M^2(v_1+v_4)^2\\
 &+\tfrac12(a-c)^2(v_1-v_4)^2.
\end{aligned}
\]
Together with $S^\top S=M^2I_2$, this gives singular values $M$, except for one equal to $a-c=M-2\beta/M\ge m$. Also,
\[
\det(\lambda I-R_4)=(\lambda^2-M^2)(\lambda^2+M^2-2\beta).
\]
Since $M^2-2\beta\ge Mm>0$, $R$ is diagonalizable, with eigenvalues $\pm M$ of even multiplicity $d-1$ and $\pm i\sqrt{M^2-2\beta}$. Hence $\det R=M^{2d-2}(M^2-2\beta)>0$. Pairing the negative eigendirections gives a real logarithm $L$ \cite{culver1966}. For $F(u)=e^{L/K}u$, we obtain $F^K=R$ and $f(x)=\beta x$.
\end{proof}

Exact reconstruction of every input is possible precisely when $M(M-m)\ge2$. This requires $M\ge2$ when $m=1$, or $\eta\ge\frac12\log3$ when $(m,M)=(e^{-\eta},e^\eta)$. The bounds concern $G$, not its roots. Write $\mathrm O(n)=\{A\in\R^{n\times n}\mid A^\top A=I\}$; $\SO(n)$ consists of its matrices with determinant one. At $m=M=1$, a coordinate swap reconstructs exactly for even $d$ and lies in $\SO(2d)$. For odd $d$, one extra hidden coordinate permits the swap followed by a sign reversal on that unused coordinate, giving an exact map in $\SO(2d+1)$.

\section{A finite-data loss bound}
At zero, $Df(0)=Dg(h(0))Dh(0)$, whereas the off-diagonal blocks of $DG(0)$ are $Dg(0)$ and $Dh(0)$.

\begin{proposition}\label{local}
Let $r=d$ with odd $d\ge3$, and let $G$ be an orientation-preserving $C^1$ diffeomorphism. Let $m_0,M_0$ be the smallest and largest singular values of $DG(0)$, and put $\delta_*=[1-M_0(M_0-m_0)/2]_+$. Suppose $Dg$ is $L_g$-Lipschitz on the segment joining $0$ and $h(0)$. Then
\begin{equation}\label{displacement}
 \|Df(0)-I\|\ge\delta_a,
 \qquad\delta_a=[\delta_*-M_0L_g\|h(0)\|]_+.
\end{equation}
If $h(0)=0$, take $\delta_a=\delta_*$ without an assumption on $Dg$. Let $X\in\R^d$ be a random vector with $\mathbb EX=0$, positive-definite covariance $\Sigma=\mathbb E[XX^\top]$, and $\mathbb E\|X\|^4<\infty$, where $\mathbb E$ denotes expectation. Write $\lambda_{\min}(\Sigma)$ for its least eigenvalue. If $Df$ is $H_f$-Lipschitz on a ball containing the support of $X$ and zero, then
\begin{equation}\label{datafloor}
 \frac{\mathbb E\|f(X)-X\|^2}{d}
 \ge\frac1d\left[
 \sqrt{\lambda_{\min}(\Sigma)}\,\delta_a
 -\frac{H_f}{2}\sqrt{\mathbb E\|X\|^4}
 \right]_+^2.
\end{equation}
\end{proposition}
\begin{proof}
Write $DG(0)=\left(\begin{smallmatrix}D_0&B_0\\C_0&T_0\end{smallmatrix}\right)$. The chain rule gives
\[
 Df(0)-B_0C_0=(Dg(h(0))-Dg(0))C_0.
\]
Use \eqref{matrixfloor} with $m_0,M_0$ and $\|C_0\|\le M_0$ to obtain \eqref{displacement}. Taylor's theorem gives
\[
 \begin{gathered}
 f(X)-X=f(0)+(Df(0)-I)X+r(X),\\
 \|r(X)\|\le\tfrac12H_f\|X\|^2.
 \end{gathered}
\]
Writing $A=Df(0)-I$ and using $\mathbb EX=0$ give
\[
\begin{aligned}
\mathbb E\|f(0)+AX\|^2
&=\|f(0)\|^2+\tr(A\Sigma A^\top)\\
&\ge\lambda_{\min}(\Sigma)\|A\|_F^2
\ge\lambda_{\min}(\Sigma)\delta_a^2.
\end{aligned}
\]
The reverse triangle inequality gives
\[
\begin{aligned}
\sqrt{\mathbb E\|f(X)-X\|^2}
&\ge \sqrt{\mathbb E\|f(0)+AX\|^2}
     -\sqrt{\mathbb E\|r(X)\|^2}\\
&\ge \sqrt{\lambda_{\min}(\Sigma)}\,\delta_a
     -\frac{H_f}{2}\sqrt{\mathbb E\|X\|^4}.
\end{aligned}
\]
Taking the positive part, squaring, and dividing by $d$
prove \eqref{datafloor}.
\end{proof}

The distribution may be uniform on centered finite data. The bound on $Df$ must hold throughout the ball, not only at sample points. Under \eqref{bounds}, $\delta_*\ge\delta_0$. If $DG$ is globally $H_G$-Lipschitz, one may take $L_g=H_G$ and $H_f=H_GM(M+1)$. The right-hand side of \eqref{datafloor} can be zero and is not generally sharp. If $G$ is odd, then $h(0)=0$; if $DG(0)$ is a rotation, then $\delta_*=1$.

For fixed $G$ with $h(0)=0$ and $Df$ Lipschitz near zero, apply \eqref{datafloor} to $X=tZ$. If $Z$ is centered, bounded, and has positive-definite covariance $\Sigma$, then
\[
 \liminf_{t\downarrow0}
 \frac{\mathbb E\|f(tZ)-tZ\|^2}{t^2\tr\Sigma}
 \ge\frac{\lambda_{\min}(\Sigma)}{\tr\Sigma}\,\delta_*^2.
\]
Without the condition on $h(0)$, exact reconstruction on a ball is possible. We use the radial rotations of Damelin and Fefferman \cite[Sec.~3.2]{damelin2024}.

\begin{theorem}\label{ball}
For every $d\ge1$, $a>0$, and $\eta>0$, an orientation-preserving $C^1$ diffeomorphism $G$ of $\R^{2d}$ exists with all Jacobian singular values in $[e^{-\eta},e^\eta]$ and
\[
 f(x)=x\qquad\text{whenever }\|x\|\le a.
\]
This construction uses $K=1$.
\end{theorem}
\begin{proof}
Put $J=\left(\begin{smallmatrix}0&-I\\I&0\end{smallmatrix}\right)$ and $A_\theta=e^{\theta J}$. Choose a $C^1$ angle $\theta$ equal to $\pi/2$ on $[0,a]$ and $-\pi/2$ on $[b,\infty)$, with
\[
 |r\theta'(r)|\le2\sinh\eta.
\]
For example, set $b=ae^L$, $L=3\pi/(4\sinh\eta)$, and
\[
 \theta(r)=\pi/2-\pi p\bigl(\log(r/a)/L\bigr),
\]
where $p(s)=3s^2-2s^3$ on $[0,1]$, extended by $0$ and $1$ on the two sides. Take $\|c\|>a+b$ and define
\begin{equation}\label{twist}
 G(u)=(-c,c)+A_{\theta(\|u\|)}u.
\end{equation}
If $\|x\|\le a$, then
\[
 G(x,0)=(-c,x+c),\qquad G(0,x+c)=(x,c).
\]
The second identity uses $\|x+c\|>b$, so reconstruction is exact.

Before translation, the map preserves radius and has inverse $v\mapsto A_{-\theta(\|v\|)}v$. For $r=\|u\|>0$,
\[
 DG(u)=A_{\theta(r)}\left(I+\frac{\theta'(r)}r Ju\,u^\top\right).
\]
The second factor is a shear with parameter $t=r\theta'(r)$ in the plane spanned by $u$ and $Ju$. Its two nonunit singular values are $\sqrt{1+t^2/4}\pm|t|/2$, which lie in $[e^{-\eta},e^\eta]$, and its determinant is one. The map is affine near zero, completing the $C^1$ assertion.
\end{proof}

For odd $d\ge3$ and $\eta<\frac12\log3$, \eqref{sharp} is positive although this construction has zero loss on the prescribed ball. Its $\|h(0)\|=\|c\|$ grows exponentially in $1/\eta$ as $\eta\downarrow0$. The construction does not supply roots at prescribed $K>1$.

\section{The nonexpansive case}
\begin{theorem}\label{nonexpansive}
If $G$ is globally $1$-Lipschitz and $f(x)=x$ for every $x\in\R^d$, then $r\ge d$ and
\begin{equation}\label{normal}
 G(x,Cz+Uw)=\bigl(z-a,\ C(x+a)+U\phi(w)\bigr),
\end{equation}
for $x,z\in\R^d$ and $w\in\R^{r-d}$, where $[C\ U]\in\mathrm O(r)$, $C\in\R^{r\times d}$, $U\in\R^{r\times(r-d)}$, $a\in\R^d$, and $\phi$ is $1$-Lipschitz from $\R^{r-d}$ to itself. Conversely, every such map is $1$-Lipschitz and reconstructs exactly. Its encoder $h(x)=C(x+a)+U\phi(0)$ is an affine isometry.
\end{theorem}
\begin{proof}
For $x,z\in\R^d$, equality holds throughout
\[
 \|x-z\|\le\|h(x)-h(z)\|\le\|GEx-GEz\|\le\|x-z\|.
\]
Put $c=h(0)$ and $q(x)=h(x)-c$. Polarization gives $\langle q(x),q(z)\rangle=\langle x,z\rangle$. With $e_j$ the coordinate vectors and $C=(q(e_1),\ldots,q(e_d))$, we have $C^\top C=I$ and
\[
 \|q(x)-Cx\|^2
 =2\|x\|^2-2\sum_jx_j\langle q(x),q(e_j)\rangle=0.
\]
Thus $r\ge d$ and $h(x)=Cx+c$. Equality of full and hidden distances makes $QG(x,0)$ constant. Exactness and the Lipschitz inequality similarly make $HG(0,Cx+c)$ constant. Write
\[
 G(s,0)=(p,Cs+c),\qquad G(0,Cs+c)=(s,b).
\]
Fix $(x,y)$ and write $G(x,y)=(u,v)$. For every $s\in\R^d$,
\[
\begin{aligned}
 \|u-p\|^2+\|v-c-Cs\|^2&\le\|x-s\|^2+\|y\|^2,\\
 \|u-s\|^2+\|v-b\|^2&\le\|x\|^2+\|y-c-Cs\|^2.
\end{aligned}
\]
Cancel $\|s\|^2$ and set $s=\pm te_j$. Letting $t\to\infty$ forces the linear coefficients to agree, which yields
\[
 C^\top(v-c)=x,\qquad u=C^\top(y-c).
\]
Choose the columns of $U$ as an orthonormal basis of the orthogonal complement of the column space of $C$. Put $a=C^\top c$ and $\Phi(x,z,w)=U^\top HG(x,Cz+Uw)$. Then
\[
G(x,Cz+Uw)=\bigl(z-a,C(x+a)+U\Phi(x,z,w)\bigr).
\]
For two triples $(x,z,w)$, let $\Delta$ denote differences of corresponding values.
Since $[C\ U]$ is orthogonal and $G$ is $1$-Lipschitz,
\[
\|\Delta x\|^2+\|\Delta z\|^2+\|\Delta\Phi\|^2
\le \|\Delta x\|^2+\|\Delta z\|^2+\|\Delta w\|^2.
\]
Taking $\Delta w=0$ shows that $\Phi(x,z,w)=\phi(w)$.
The same inequality gives $\operatorname{Lip}(\phi)\le1$,
proving \eqref{normal}. Conversely, orthogonality and
$\operatorname{Lip}(\phi)\le1$ make \eqref{normal}
$1$-Lipschitz, and $g(h(x))=C^\top h(x)-a=x$.
\end{proof}

\section{Experiments}

\textbf{Numerical checks.} We evaluated the attaining matrices for $d=3,5,7,9$ and 20 pairs $(m,M)$ per dimension. The largest difference from \eqref{sharp} was $1.12\times10^{-16}$ in float64. We evaluated \eqref{twist} on 10,000 unit-ball inputs for each of $d=3,5$ and $\eta=0.25,0.35,0.50$. No singular-value bound was violated. Table~\ref{twisttable} gives the $d=3$ results.

\begin{table}[b]
\centering\small
\setlength{\tabcolsep}{4pt}
\caption{The bound \eqref{sharp} and the largest reconstruction error over 10,000 unit-ball points.}
\label{twisttable}
\begin{tabular}{cccc}
\toprule
$\eta$ & $\delta_0$ & $\|h(0)\|$ & Ball error \\
\midrule
$0.25$ & $0.67564$ & $11242.90$ & $9.10\times10^{-13}$ \\
$0.35$ & $0.49312$ & $734.51$ & $5.69\times10^{-14}$ \\
$0.50$ & $0.14086$ & $93.98$ & $7.11\times10^{-15}$ \\
\bottomrule
\end{tabular}
\end{table}

\noindent \textbf{Finite-data experiment.} We use the three coordinates of a 798,452-point forest scan \cite{medina2025}. Each horizontal median quadrant is held out once. Training excludes points within $0.81629$ coordinate units of the closed test quadrant \cite{roberts2017}. Each split has about 512,000 to 522,000 training points and 199,000 to 200,000 test points. Standardization and centering use training coordinates only, followed by multiplication by $t\in\{0.05,0.2\}$.

We use the zero-bias map
\begin{equation}\label{contractingstep}
 \begin{aligned}
 F_K(u)&=\psi_K(Au),\qquad A=e^S,\quad S^\top=-S,\\
 \psi_K(z)&=(1-0.4/K)z+(0.4/K)\tanh z.
 \end{aligned}
\end{equation}
Here $\tanh$ acts coordinatewise and $G=F_K^K$. The map $F_K$ is globally $1$-Lipschitz, every $DG(u)$ has smallest singular value at least $(1-0.4/K)^K$, and $DG(0)=A^K$ is a rotation. Thus $h(0)=0$ and $\delta_*=1$ at $r=d=3$, for every trained weight matrix. Since $\sup_{s\in\R}|\tanh''s|=4/(3\sqrt3)$, composition gives
\[
 H_G\le\frac{1.6}{3\sqrt3},\qquad H_f\le\frac{3.2}{3\sqrt3}.
\]
We use these bounds in \eqref{datafloor} on each full centered training region.

The 96 runs use $r=3,4$, $K=1,5$, two scales, three seeds, and four splits. Initial linear maps retain the two leading principal directions at $r=3$ and reconstruct exactly at $r=4$, using real $K$th roots and seeded hidden bases. We minimize reconstruction error on the same fixed 2,048-point subset per split, using analytic derivatives and at most 100 L-BFGS iterations. All final iterates are retained, including 75 that reached the limit.

\begin{table}[t]
\centering\small
\setlength{\tabcolsep}{4pt}
\caption{Normalized full-training errors, averaged over three seeds within each fold, then equally over four folds. The bound applies to $r=3$.}
\label{anchortable}
\begin{tabular}{ccccc}
\toprule
$t$ & $K$ & Bound & $r=3$ & $r=4$ \\
\midrule
$0.05$ & $1$ & $0.15501$ & $0.18498$ & $5.89\times10^{-6}$ \\
$0.05$ & $5$ & $0.15501$ & $0.18497$ & $3.51\times10^{-6}$ \\
$0.20$ & $1$ & $0.08253$ & $0.18564$ & $1.09\times10^{-3}$ \\
$0.20$ & $5$ & $0.08253$ & $0.18537$ & $6.25\times10^{-4}$ \\
\bottomrule
\end{tabular}
\end{table}

We divide mean-square error per coordinate by the full training variance per coordinate. The bound is positive in all 48 equal-width models and, at scale $0.05$, averages $84\%$ of the mean training error. For $K=5$, adding one hidden coordinate changes mean normalized test error from $0.81734$ to $1.47\times10^{-5}$ at scale $0.05$, and from $0.81775$ to $0.00266$ at scale $0.2$. The bound is evaluated only on centered training data. The controls already have small initial error; the experiment establishes neither a depth gain nor a new optimization method.

\noindent\textbf{Acknowledgments.} The authors thank Randy Paffenroth and Shiquan He for discussions and the implementation, and Jonathan Batchelor and L. Monika Moskal (University of Washington) for the scan.


\begin{center}\textbf{Compliance with Ethical Standards}\end{center}
This work uses existing datasets and computational models. It involves no new human or animal experiments. Claude assisted with code and drafting the manuscript. The authors are responsible for the mathematics, ideation of experiments and the text.
\end{document}